\documentclass[letterpaper, 10pt]{ieeeconf}
\IEEEoverridecommandlockouts
\usepackage{cite}
\usepackage{amsmath,amssymb,amsfonts}
\usepackage{bm}
\usepackage{algorithmic}
\usepackage{graphicx}
\usepackage{textcomp}
\usepackage{xcolor,color}
\usepackage[ruled,vlined]{algorithm2e}
\usepackage{hyperref}
\usepackage{tikz}
\usetikzlibrary{positioning}
    
\makeatletter
\let\MYcaption\@makecaption
\makeatother 

\usepackage[font=footnotesize]{subcaption}

\makeatletter
\let\@makecaption\MYcaption
\makeatother

\newtheorem{definition}{Definition}

\newtheorem{theorem}{Theorem}
\newtheorem{proposition}{Proposition}
\newtheorem{remark}{Remark}

\newtheorem{assumption}{Assumption}

\DeclareMathOperator{\rank}{rank}
\DeclareMathOperator{\trace}{trace}

\newcommand{\q}{\bm{q}}
\newcommand{\qd}{\dot{\q}}
\newcommand{\qdd}{\ddot{\q}}
\newcommand{\btau}{\bm{\tau}}

\newcommand{\M}{\bm{M}}
\newcommand{\C}{\bm{C}}
\newcommand{\J}{\bm{J}}

\newcommand{\f}{\bm{f}}
\newcommand{\LL}{\bm{\Lambda}}
\newcommand{\Jbar}{\bar{\J}}
\newcommand{\Q}{\bm{Q}}

\newcommand{\x}{\bm{x}}
\newcommand{\xd}{\dot{\x}}
\newcommand{\xdd}{\ddot{\x}}

\renewcommand{\u}{\bm{u}}

\newcommand{\xerr}{\tilde{\x}}
\newcommand{\xderr}{\dot{\tilde{\x}}}
\newcommand{\xdderr}{\ddot{\tilde{\x}}}
\newcommand{\K}{\bm{K}}

\newcommand{\xx}{\mathbf{x}}
\newcommand{\yy}{\mathbf{y}}
\newcommand{\uu}{\mathbf{u}}

\begin{document}

\title{Control Barrier Functions for Singularity Avoidance \\ in Passivity-Based Manipulator Control}

\author{Vince Kurtz, Patrick M. Wensing, and Hai Lin
\thanks{
This work was supported in part by the National Science Foundation under Grants IIS-1724070, CNS-1830335, IIS-2007949, and CMMI-1835186. 
}\thanks{
Vince Kurtz and Hai Lin are with the Department of Electrical Engineering,
University of Notre Dame, Notre Dame, IN 46556 USA (e-mail: \texttt{vkurtz@nd.edu;
hlin1@nd.edu}).
}\thanks{
Patrick M. Wensing is with the Department of Aerospace and Mechanical
Engineering, University of Notre Dame, Notre Dame, IN 46556 USA (e-mail:
\texttt{pwensing@nd.edu}).
}
}

\maketitle
\thispagestyle{empty}

\begin{abstract}
    Task-space Passivity-Based Control (PBC) for manipulation has numerous appealing properties, including robustness to modeling error and safety for human-robot interaction. Existing methods perform poorly in singular configurations, however, such as when all the robot's joints are fully extended. Additionally, standard methods for constrained task-space PBC guarantee passivity only when constraints are not active. We propose a convex-optimization-based control scheme that provides guarantees of singularity avoidance, passivity, and feasibility. This work paves the way for PBC with passivity guarantees under other types of constraints as well, including joint limits
    and contact/friction constraints. The proposed methods are validated in simulation experiments on a 7 degree-of-freedom manipulator. 
\end{abstract}

\section{Introduction and Related Work}\label{sec:intro}

Passivity is a highly desirable property for robots collaborating with humans or working in delicate and uncertain environments. Not only do passivity-based controllers tend to be robust to modeling errors and disturbances \cite{folkertsma2017energy}, they also provide guarantees not available to other nonlinear controllers, as the feedback interconnection of passive systems is always passive \cite{spong2006robot}. 

Passivity-Based Control (PBC) has been successfully applied in many areas of robotics, including manipulation \cite{takegaki1981new,stramigioli1999passivity,albu2007unified}, legged locomotion \cite{henze2016passivity,mesesan2019dynamic,englsberger2020mptc,kurtz2020approximate}, and even autonomous driving \cite{rahnama2016passivation}. In the manipulation context, PBC is most naturally formulated in terms of task-space control \cite{englsberger2020mptc}, where the goal is to track a task-space (e.g., end-effector) reference. Like many task-space strategies, however, task-space PBC performs poorly in near-singular configurations like those in Figure \ref{fig:singular_configs}.

In this paper, we present an optimization-based task-space PBC strategy that guarantees both passivity and singularity avoidance. Our key insights are to constrain the evolution of the storage function in a convex quadratic program (QP) and to track a reference system rather than a predetermined reference trajectory. This allows us to modify the input to the reference system when necessary to avoid singular configurations.  

\begin{figure}
    \centering
    \begin{subfigure}{0.45\linewidth}
        \centering
        \includegraphics[width=0.9\linewidth]{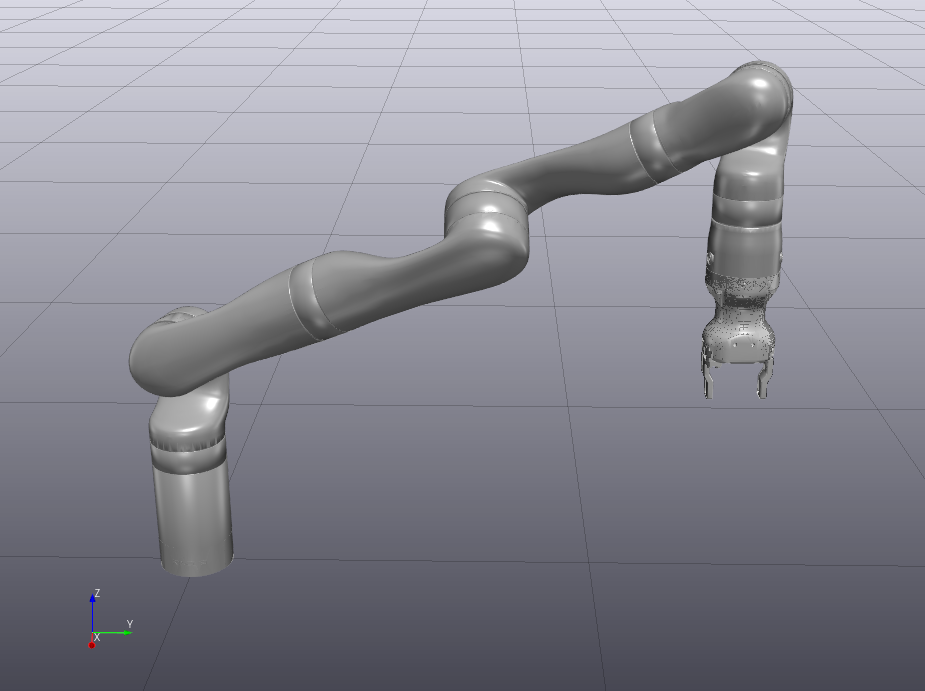} 
        \caption{}
        \label{fig:singular_config_1}
    \end{subfigure}
    \begin{subfigure}{0.45\linewidth}
        \centering
        \includegraphics[width=0.9\linewidth]{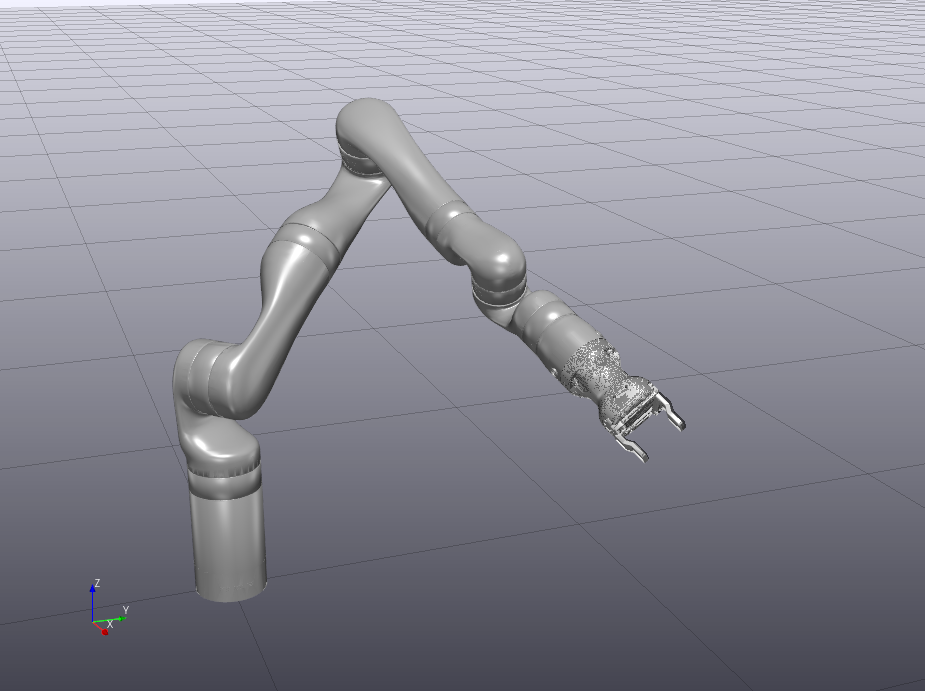} 
        \caption{}
        \label{fig:singular_config_2}
    \end{subfigure}
    \begin{subfigure}{0.9\linewidth}
        \centering
        \includegraphics[width=0.85\linewidth]{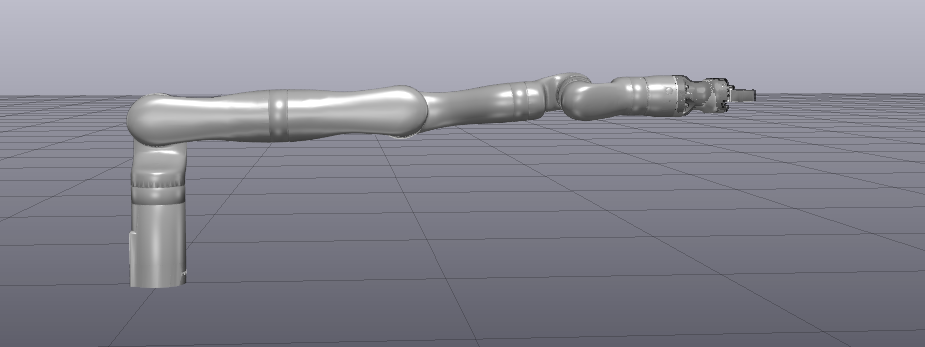} 
        \caption{}
        \label{fig:singular_config_3}
    \end{subfigure}
    \caption{Examples of singular configurations in the Kinova Gen3 robot arm. We present a method of avoiding such singularities while maintaining passivity guarantees by modifying the end-effector reference when necessary.}
    \label{fig:singular_configs}
\end{figure}

Common strategies for singularity avoidance include restricting the task-space reference \cite{marani2002real,kim2004general}, carefully selecting joint limits such that singularities are avoided \cite{cortez2019control}, and using a damped pseudoinverse of the task-space Jacobian \cite{wampler1986manipulator,henze2016passivity,mesesan2019dynamic,englsberger2020mptc}. Restricting the task-space reference a-priori can be difficult, since singular configurations can occur even when the end-effector is well within the reachable workspace (see Figure~\ref{fig:singular_configs}). Similarly, determining joint limits that exclude all possible singularities can be severely restrictive, especially for high degree-of-freedom (DoF) robots. Using a damped Jacobian pseudoinverse avoids many of the numerical issues associated with singular configurations, but at the price of losing formal stability guarantees and degraded performance. 

Of particular interest to us are control strategies based on the manipulability index \cite{yoshikawa1985manipulability}, which provides a smooth measure of how close a given configuration is to a singularity. Such strategies typically focus on maximizing the manipulability index in the null-space of a task-space controller \cite{nemec2000null,jin2017manipulability,su2019manipulability}. These null-space strategies use redundant degrees of freedom to keep the robot as far away as possible from singularity while matching a given end-effector reference. These techniques cannot guarantee singularity avoidance, however, as the task-space reference can always be chosen such that encountering a singularity is inevitable. For example, null-space strategies might be able to avoid the singularity shown in Figure \ref{fig:singular_config_2}, but not those shown in Figure \ref{fig:singular_config_1} and \ref{fig:singular_config_3}.

We propose an alternative approach to singularity avoidance, designing a Control Barrier Function (CBF) that keeps the manipulability index above a given threshold. This allows us to encode singularity avoidance as a linear constraint within a convex Quadratic Program (QP). This approach is closely related to \cite{marani2002real,kim2004general}, which use the manipulability index to avoid singularities in kinematic motion planning. 

Adding CBF constraints to a QP-based controller presents additional challenges, however. Specifically, existing methods for constrained passivity-based control can only guarantee passivity when the additional constraints (including singularity avoidance constraints as well as joint limits and friction/contact constraints) are not active \cite{henze2016passivity,mesesan2019dynamic,englsberger2020mptc}. This leads to our second contribution: we propose a QP-based controller that guarantees both passivity and constraint satisfaction. This QP is always feasible. Our key insight in this regard is to treat the reference signal as a system itself, the input of which can be adjusted to ensure feasibility. This is similar in spirit to the idea of reference governors \cite{kolmanovsky2014reference}.

Our primary contributions can be summarized as follows:
\begin{enumerate}
    \item We propose a CBF that enforces a minimum manipulability index. This CBF is a true CBF, i.e., the associated linear constraints are guaranteed to always be feasible.
    \item We propose a new QP-based strategy for task-space PBC that guarantees constraint satisfaction, passivity, and feasibility.
\end{enumerate}

To the best of our knowledge, this is the first task-space PBC strategy that can guarantee all three of these properties (passivity, constraint satisfaction, and feasibility).

The remainder of this paper is organized as follows: problem formulation and background information are presented in Section \ref{sec:background}. Main results are presented in Section \ref{sec:main_results}, and are supported by simulation experiments in Section \ref{sec:simulation}. We conclude with Section \ref{sec:conclusion}. Additionally, we provide an online interactive demonstration to accompany the paper \cite{colab}.

\section{Background}\label{sec:background}

\subsection{Problem Formulation}

In this paper, we consider torque control of a rigid manipulator arm, the dynamics of which can be written as
\begin{equation}\label{eq:wholebody_dynamics}
   \M(\q)\qdd + \C(\q,\qd)\qd + \btau_g(\q) = \btau,
\end{equation}
where $\q \in \mathbb{R}^n$ are joint positions, $\M$ is the positive-define mass matrix, $\C$ is the Coriolis matrix, $\btau_g$ are gravitational torques, and $\btau$ are applied control torques. A broad class of rigid-body systems has dynamics of this form \cite{featherstone2014rigid}. 

We assume that the applied torques $\btau$ are unconstrained:
\begin{assumption}\label{assumption:torque_limits}
    Any joint torques $\btau \in \mathbb{R}^n$ can be applied to the robot, i.e., $\btau$ is not bounded. 
\end{assumption}
Modern torque-controlled robots typically have high torque limits, making this assumption reasonable in practice. 

Rather than controlling joint angles $\q$ directly, we are interested in regulating the task-space state $\x(\q) \in \mathbb{R}^m$ of the robot, where $m \leq n$. In the manipulation context, the task-space is typically the end-effector pose (position and orientation%
\footnote{Note that the orientation manifold $SO(3)$ is non-Euclidean and so the notation $\x \in \mathbb{R}^m$ is slightly inconsistent in this case. For simplicity of presentation, however, we assume for the remainder of this paper that orientations are expressed as Euler angles and write orientations as belonging to $\mathbb{R}^3$. The proposed techniques generalize naturally to nonsingular orientation representations such as rotation matrices or quaternions.}%
), but the task-space can be any smoothly-varying quantity of interest, such as the center-of-mass or a particular link's position.

The task-space differential kinematics are characterized by the Jacobian
\begin{equation}\label{eq:jacobian}
   \J = \frac{\partial \x(\q)}{\partial \q},
\end{equation}
where $\xd = \J\qd$. When the Jacobian is full row-rank, i.e., $\rank(\J) = m$, joint velocities $\qd$ can be selected to correspond to any desired task-space velocity $\xd$. When the Jacobian is not full rank, however, certain task-space velocities may not be achievable, and we say the robot is in a \textit{singular configuration}. Examples of singular configurations are shown in Figure \ref{fig:singular_configs}. 

As discussed in Section \ref{subsec:tpbc} below, conventional formulations of task-space passivity-based control break down when the robot enters a singular or near-singular configuration. For this reason, we focus on control strategies that guarantee singularity avoidance. 

To maintain guarantees of both passivity and singularity avoidance, we consider the problem of tracking a \textit{reference system} rather than (as is typically the case) a reference trajectory. We assume that this reference system is governed by simple double integrator dynamics,
\begin{equation}\label{eq:reference_system}
    \xdd_r = \u_r,
\end{equation}
where $\x_r \in \mathbb{R}^m$ is a reference task-space state (e.g., end-effector pose) and $\u_r$ is the input to the reference system. 

We assume that a nominal reference input $\u_r^{nom}$ is available from a higher-level controller. We define position and velocity tracking errors as follows:
\begin{gather}
    \xerr = \x(\q) - \x_r, \\
    \xderr = \xd(\q,\qd) - \xd_r.
\end{gather}

With this in mind, our goal is to formulate a controller with the following properties:
\begin{enumerate}
    \item The closed-loop system is passive,
    \item The system remains singularity free ($\rank(\J(\q)) = m$).
\end{enumerate}

\subsection{Passivity}\label{subsec:passivity}

In this section, we present a formal definition of passivity and highlight some of the advantages of PBC. 

Consider a state-space dynamical system
\begin{equation}\label{eq:general_system}
    \Sigma = 
    \begin{cases}
        \dot{\xx} = f(\xx,\uu) \\
        \yy = g(\xx,\uu)
    \end{cases} 
\end{equation}
where $\xx$ is the system state, $\uu$ is the input, and $\yy$ is the output. Assume that $f : \mathbb{R}^k \times \mathbb{R}^p \to \mathbb{R}^k$ is locally Lipschitz and $g : \mathbb{R}^k \times \mathbb{R}^p \to \mathbb{R}^p$ is continuous. Note that system (\ref{eq:wholebody_dynamics}) can be written in this form with $\xx = [\q^T~\qd^T]^T$.

\begin{definition}
    The system (\ref{eq:general_system}) is said to be passive with input $\uu$ and output $\yy$ if there exists a continuously differentiable positive semidefinite storage function $V(\xx)$ such that
    \begin{equation}\label{eq:general_passivity_property}
        \dot{V} = \frac{\partial V}{\partial \xx}f(\xx,\uu) \leq \yy^T\uu, ~~~\forall~ (\xx,\uu) \in \mathbb{R}^k \times \mathbb{R}^p.
    \end{equation}
    \vspace{-1.0em} 
\end{definition}

Besides having close connections with other notions of nonlinear system stability \cite{khalil2002nonlinear}, passivity has important implications for safety and robustness. Often, the storage function $V$ represents some sort of system energy. In this case, the supply rate $\yy^T\uu$ corresponds to a power input port and the passivity property (\ref{eq:general_passivity_property}) states that the system energy can be increased only through this port \cite{folkertsma2017energy}. 

The following properties highlight the advantages of PBC in robotics. First, we have the well-known passivity theorem:
\begin{theorem}[\cite{khalil2002nonlinear}]
    The feedback interconnection of two passive systems is passive. 
\end{theorem}
This property does not hold for other notions of nonlinear system stability. Furthermore, any stable non-passive system can become unstable via interconnection with a passive system:

\begin{theorem}[\cite{folkertsma2017energy}]
    Given any non-passive system $\Sigma_1$, there always exists a passive system $\Sigma_2$ that gives rise to unbounded behavior of the feedback interconnection of $\Sigma_1$ and $\Sigma_2$.  
\end{theorem}

In the manipulation context, $\Sigma_2$ might consist of objects in the environment, humans, or other robots. This theorem states that if the controlled robot is not passive, a passive environment can always be constructed that destabilizes the system. For this reason, passivity is often considered a necessary condition for safety \cite{folkertsma2017energy,henze2016passivity,mesesan2019dynamic,englsberger2020mptc}. 

\subsection{Task-Space Passivity-Based Control}\label{subsec:tpbc}

In this section, we introduce the standard method of (unconstrained) task-space passivity-based control. This method is typically presented in terms of tracking a reference trajectory rather than a reference system: note that these two formulations are equivalent if we assume $\u_r$ to be fixed. 

Task-space passivity-based control begins with the definition of a storage function,
\begin{equation}\label{eq:storage function}
    V(\q,\qd) = \frac{1}{2}\xderr^T\LL\xderr + \frac{1}{2}\xerr^T\K_P\xerr,
\end{equation}
where $\LL = (\J\M^{-1}\J^T)^{-1}$ is the task-space inertia matrix and $\K_P = \K_P^T$ is a positive definite gain matrix. Note that this storage function can be thought of as the energy in the error between the actual task-space state and the reference system. The first term captures kinetic energy (task-space mass $\LL$ times error velocity $\xderr$ squared) and the second term is a potential energy. 

Differentiating the storage function, we have
\begin{align}
    \dot{V} &= \frac{1}{2}\xderr^T\dot{\LL}\xderr + \xderr^T\LL\xdderr + \xderr^T\K_P\xerr \\
    &= \xderr^T\left( \LL\Q\Jbar\xderr + \LL(\xdd - \xdd_r) + \K_P\xerr \right) \\
    &= \xderr^T\left( \LL\Q\Jbar\xderr - \LL\xdd_r + \LL(\J\qdd + \dot{\J}\qd) + \K_P\xerr \right),\label{eq:vdot}
\end{align}
where $\Q = \J\M^{-1}\C - \dot{\J}$, $\Jbar = \M^{-1}\J^T\LL$ is a pseudoinverse of $\J$, and $\dot{\LL} - 2\LL\Q\Jbar$ is skew-symmetric \cite{spong2006robot}. 

If task-space forces $\f = \Jbar^T\btau$ are chosen such that
\begin{equation}\label{eq:controller}
    \f = \LL\xdd_r + \Jbar^T\btau_g + \LL\Q(\qd - \Jbar\xderr) - \K_P\xerr - \K_D\xderr,
\end{equation}
where $\K_D = \K_D^T$ is a positive definite damping matrix, then
\begin{equation}
    \dot{V} = - \xderr^T \K_D \xderr \leq 0.
\end{equation}

Furthermore, if the system is subject to external disturbances $\btau_{ext}$, i.e., the system dynamics are $\M\qdd + \C\qd + \btau_g = \btau + \btau_{ext}$, then
\begin{equation}\label{eq:passivity_wrt_disturbance}
   \dot{V} = - \xderr^T \K_D \xderr + \xderr^T \Jbar^T \btau_{ext} \leq \xderr^T \f_{ext},
\end{equation}
where $\f_{ext} = \Jbar^T\btau_{ext}$ are the task-space forces from $\btau_{ext}$.

This passivity property characterizes the energy of the closed-loop system. Noting that $\xderr^T\f_{ext}$ has units of power,  (\ref{eq:passivity_wrt_disturbance}) characterizes how the energy $V(\q,\qd)$ increases (due to disturbances $\f_{ext}$) and decreases (due to the dissipation $-\xderr^T\K_D\xderr$).

Despite these advantages, the method presented above has two important limitations: (1) it is difficult to include constraints on the system while maintaining passivity and (2) tracking performance degrades rapidly in near-singular configurations, since joint torques $\btau$ must become extremely large to be consistent with (\ref{eq:controller}) . 

Most robotic systems are subject to constraints, e.g., joint angle and velocity limits, and friction and contact constraints. A standard approach for addressing such constraints in the context of passivity-based control is to replace the closed-form controller (\ref{eq:controller}) with the solution of a convex QP that imposes such constraints and is re-solved at each timestep. This QP takes the form
\begin{equation}\label{eq:standard_pbc_qp}
\begin{aligned}
    \min_{\qdd,\btau} ~& \|\Jbar^T\btau - \f^{des}\|^2 \\
    \text{s.t. } & \M\qdd + \C\qd + \btau_g = \btau \\
                 & \text{Additional Constraints},
\end{aligned}
\end{equation}
where $\f^{des}$ are desired task-space forces given by (\ref{eq:controller}) and ``Additional Constraints'' might include contact, joint limit, or other constraints \cite{cortez2020correct,henze2016passivity,mesesan2019dynamic}. If $\Jbar^T\btau = \f^{des}$ (i.e., the additional constraints are not active), then all of the passivity properties outlined above hold. If these constraints become active, however, the controller will attempt to match $\f^{des}$ but any guarantees of passivity (and associated safety and robustness) are lost. 

In this paper, we will present an alternative optimization-based method for including constraints, where passivity is guaranteed for \textit{any} solution to the optimization, even when additional constraints are active. 

The second issue is singular configurations. In a singular configuration, $\J$ is no longer full-rank, and arbitrary task-space forces $\f$ cannot be applied. In the case of a convex optimization formulation like (\ref{eq:standard_pbc_qp}), this corresponds to the optimization problem becoming ill-conditioned. In practice, this means that extreme joint torques are applied as the robot approaches a singular configuration.

A common method of avoiding this issue is to use a damped pseudoinverse of the Jacobian 
\begin{equation}\label{eq:damped_jacobian_pinv}
    \J^T(\J\J^T + \delta \bm{I})^{-1}  
\end{equation}
instead of $\Jbar$, where $\delta > 0$ is a small constant \cite{wampler1986manipulator}. This damped pseudoinverse prevents the numerical conditioning issues associated with singular configurations, but passivity can no longer be formally guaranteed and tracking performance degrades near singularities \cite{henze2016passivity}. 

In this paper, we use CBFs to design a controller that allows the robot to avoid singular and near-singular configurations altogether while maintaining passivity guarantees. 

\subsection{Manipulability Index}\label{subsec:manipulability}

To avoid singularities, we need a smoothly varying measure of how close the given configuration $\q$ is to a singular configuration. The \textit{manipulability index} \cite{yoshikawa1985manipulability} provides such a measure. The manipulability index $\mu(\q)$ is defined as
\begin{equation}\label{eq:manipulability}
    \mu(\q) = \sqrt{\det \left[\J(\q) \J(\q)^T\right]},
\end{equation}
and is zero if $\q$ is singular and positive otherwise. 

To understand the manipulability index, note that $\mu$ can also be written as
\begin{equation*}
    \mu(\q) = \sigma_1\sigma_2\dots\sigma_m,
\end{equation*}
where $\sigma_i$ are the singular values of $\J$ (i.e., diagonal elements of $\bm{\Sigma}$ where $\J = \bm{U}\bm{\Sigma}\bm{V}^T$). If $\J$ is not full rank, then at least one $\sigma_i = 0$ and so $\mu = 0$.

Importantly, the manipulability index $\mu(\q)$ varies smoothly with $\q$ over $\mu(\q) > 0$. Indeed, we can think of the manipulability as a unique sort of task-space characterized by its own Jacobian \cite{marani2002real}
\begin{equation}
    \J_\mu(\q) = \frac{\partial \mu(\q)}{\partial \q},
\end{equation}
such that $\dot{\mu} = \J_{\mu} \qd$. The elements of this Jacobian are
\begin{equation}\label{eq:mu_gradient}
    \frac{\partial \mu(\q)}{\partial q_i} = \mu(\q) \trace\left[ \frac{\partial \J}{\partial q_i} \J^\dagger \right],
\end{equation}
where $q_i$ are the elements of $\q$ and $\J^{\dagger} = \J^T(\J \J^T)^{-1}$ is the Moore-Penrose pseudoinverse of $\J$ \cite{marani2002real}.

\subsection{Exponential Control Barrier Functions}\label{subsec:ecbf}

Control Barrier Functions are a way of designing (linear) constraints that ensure forward invariance of a safe set\footnote{While most CBF formulations focus on control-affine systems, we restrict our presentation to systems of the form (\ref{eq:wholebody_dynamics}) for simplicity.} $\mathcal{C} = \{ \q,\qd \mid h(\q,\qd) \geq 0 \}$ \cite{ames2019control}.

The basic idea is relatively simple: $\dot{h} \geq 0$ everywhere on the boundary of $\mathcal{C}$ is sufficient for forward invariance of $\mathcal{C}$ \cite{ames2019control}. In the case that $h$ has relative degree 1 ($\dot{h}$ is a function of $\qdd$, and thus implicitly also of $\btau$) a constraint of the form 
\begin{equation}
    \dot{h}(\q,\qd,\qdd) \geq - \alpha(h(\q,\qd))
\end{equation}
enforces forward invariance of $\mathcal{C}$, where $\alpha(\cdot)$ is any class-$\mathcal{K}$ function\footnote{A function $\alpha : \mathbb{R}^+ \to \mathbb{R}^+$ is in class-$\mathcal{K}$ if it is continuous, strictly increasing, and $\alpha(0) = 0$.}. For systems of the form (\ref{eq:wholebody_dynamics}), such constraints are linear in $\qdd$ and can be included in a QP like (\ref{eq:standard_pbc_qp}). 

If $h$ has higher relative degree, i.e., $\qdd$ enters in the higher-order derivatives of $h$, we need to use Exponential CBFs \cite{nguyen2016exponential} to ensure forward invariance of $\mathcal{C}$. For example, consider the case when $h$ is a function of $\q$ only. In this case, a sufficient condition for forward invariance of $\mathcal{C}$ is given by
\begin{equation}\label{eq:ecbf}
   \ddot{h}(\q,\qd,\qdd) \geq - \K_\alpha \begin{bmatrix} h(\q) \\ \dot{h}(\q,\qd) \end{bmatrix},
\end{equation}
where $\K_\alpha$ is a gain matrix that must satisfy certain regulatory conditions \cite[Theorem 2]{nguyen2016exponential}. The constraint (\ref{eq:ecbf}) is also linear in $\qdd$ for systems of the form (\ref{eq:wholebody_dynamics}), and can be enforced via QP.

The key issue when designing CBFs, regardless of relative degree, is the feasibility of the constraint (\ref{eq:ecbf}). For an arbitrary candidate barrier function $h$, there is no guarantee that a $\qdd$ satisfying (\ref{eq:ecbf}) can always be found. In this sense, CBFs are analogous to Control Lyapunov Functions: if we can find a true barrier function (or Lyapunov function) this is a useful and powerful result, but we cannot presume to use just any function as a barrier function (or a Lyapunov function). 

For further details on CBFs and ECBFs, we refer the interested reader to \cite{ames2019control} and references therein. 

\section{Main Results}\label{sec:main_results}

In this section we present our main results. First, we show how the manipulability index can be used to formulate a (true) ECBF that guarantees singularity avoidance. Second, we present a convex optimization-based controller that allows us to enforce these singularity avoidance constraints while maintaining passivity guarantees. 

\subsection{Control Barrier Functions For Singularity Avoidance}

In this section, we show how the manipulability index (\ref{eq:manipulability}) can be used to construct (linear) ECBF singularity avoidance constraints on the system (\ref{eq:wholebody_dynamics}).

Recalling that $\mu(\q) = 0$ only when the robot is in a singular configuration and $\mu(\q) > 0$ otherwise, we propose the barrier function
\begin{equation}\label{eq:our_bf}
    h(\q) = \mu(\q) - \epsilon,
\end{equation}
where $\epsilon > 0$ is a user-determined constant parameter characterizing the minimum ``distance'' to maintain from any singularities. Clearly, $h(\q) \geq 0 \implies \mu(\q) > 0 \implies \q$ is nonsingular. 

As discussed in Section \ref{subsec:manipulability} above, $\mu(\q)$, and thus also $h(\q)$, is smooth and of relative degree 2. This allows us to constrain the system to the set of states such that $\mu(\q) \geq \epsilon$ by applying the linear constraint
\begin{equation}\label{eq:our_cbf_constraint}
   \ddot{h}(\q,\qd,\qdd) \geq - \K_\alpha \begin{bmatrix} h(\q) \\ \dot{h}(\q,\qd) \end{bmatrix},
\end{equation}
where
\begin{gather*}
    \dot{h}(\q,\qd) = \J_\mu \qd, \\
    \ddot{h}(\q,\qd,\qdd) = \J_\mu \qdd + \dot{\J}_\mu \qd,
\end{gather*}
and $\K_\alpha$ satisfies the conditions of \cite[Theorem 2]{nguyen2016exponential} (closed-loop system matrix Hurwitz and total negative). 

This $h(\q)$ is a true ECBF, meaning (\ref{eq:our_cbf_constraint}) always has a solution for any $\q,\qd$ such that $h(\q) \geq 0$, as shown in the following proposition:
\begin{proposition}\label{prop:ecbf}
    The barrier function (\ref{eq:our_bf}) is an ECBF for the system (\ref{eq:wholebody_dynamics}). 
\end{proposition}
\begin{proof}
    We will show that for any configuration $\q$ and velocity $\qd$, we can select joint accelerations $\qdd$ such that (\ref{eq:our_cbf_constraint}) holds. Note that under Assumption \ref{assumption:torque_limits}, joint torques $\btau$ can always be selected to be consistent with any desired accelerations $\qdd^{des}$.
    
    With $\K_\alpha = [\alpha_1 ~ \alpha_2]$, we can write (\ref{eq:our_cbf_constraint}) as
    \begin{gather*}
        \ddot{h}(\q,\qd,\qdd) \geq -\alpha_1 h(\q) - \alpha_2 \dot{h}(\q,\qd), \\
        \J_\mu\qdd + \dot{\J}_\mu\qd \geq - \alpha_1 (\mu(\q) - \epsilon) - \alpha_2 \J_\mu \qd, \\
        \J_\mu\qdd \geq b
    \end{gather*}
    where $b = - \dot{\J}_\mu\qd - \alpha_1 (\mu(\q) - \epsilon) - \alpha_2 \J_\mu \qd$.
    
    Since $\J_\mu$ is a nonzero $1 \times n$ matrix (\ref{eq:mu_gradient}), $\qdd$ can be selected such that $\J_\mu \qdd \geq b$ for any $b \in \mathbb{R}$. Thus the ECBF constraint (\ref{eq:our_cbf_constraint}) is always feasible and the proposition holds. 
\end{proof}

With this in mind, we can guarantee singularity avoidance by including (\ref{eq:our_cbf_constraint}) as a constraint in an optimization-based controller like (\ref{eq:standard_pbc_qp}). Once the singularity avoidance constraints become active, however, any passivity guarantees are lost. In the following section, we provide an alternative optimization-based control scheme that guarantees both singularity avoidance and passivity. 

\begin{figure}
    \centering
    \begin{tikzpicture}
        \node[draw, rectangle, minimum width=2.5cm, minimum height=1.5cm] (qp) at (0,0) {QP (\ref{eq:our_qp})};
        \node[draw, rectangle, below=0.3cm of qp, minimum width=3.5cm,  minimum height=0.8cm] (robot) {Robot (\ref{eq:wholebody_dynamics})};
        \node[draw, rectangle, above=0.5cm of qp, minimum width=3.5cm, minimum height=0.8cm] (ref) {Reference System (\ref{eq:reference_system})};
        \node[draw, rectangle, above left=-0.4cm and 1.4cm of qp] (ctrl) {\begin{tabular}{c}
             Reference \\
             Controller 
        \end{tabular}};
        
        \draw[->,thick] ([yshift=0.2cm]qp.east) -- ++(1,0) |- node[pos=0.25,right] {$\u_r$} (ref.east);
        \draw[->,thick] ([yshift=-0.2cm]qp.east) -- ++(1,0) |- node[pos=0.25,right] {$\btau$} (robot.east);
        \draw[->,thick] (robot.west) -- ++(-0.5,0) |- node[pos=0.3, left] {$\q,\qd$} ([yshift=-0.3cm]qp.west);
        \draw[->,thick] (ctrl.south) |- node[pos=0.5,below] {$\u_r^{nom}$} ([yshift=0.0cm]qp.west);
        \draw[->,thick] (ref.west) -| node[pos=0.3,above] {$\x_r,\xd_r$} (ctrl.north);
        \draw[->,thick] (ref.west) -- ++ (-0.5,0) |- ([yshift=0.3cm]qp.west);
    \end{tikzpicture}
    \caption{Diagram outlining our proposed control approach. Rather than tracking a reference trajectory, we track a reference system, the input of which is modified when necessary to enforce passivity and constraint satisfaction.}
    \label{fig:control_flow}
\end{figure}
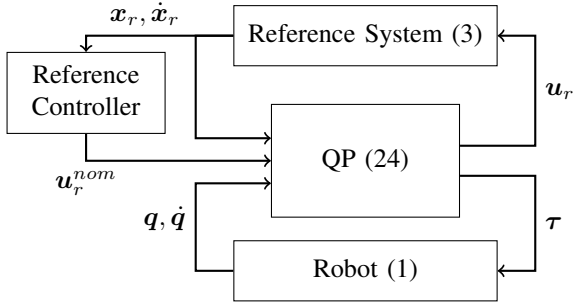

\subsection{Passivity-Guaranteed Optimization-Based Control}

In this section, we present an alternative control strategy to the standard optimization for constrained PBC (\ref{eq:standard_pbc_qp}) that guarantees passivity even when singularity avoidance constraints are active. Our key insight is to use a reference system rather than a reference trajectory, and to modify the input to this reference system when necessary for constraint satisfaction. 

A block diagram outline of this control scheme is shown in Figure \ref{fig:control_flow}. A reference controller provides a nominal input to the reference system, $\u_r^{nom}$. Our QP-based controller then takes this nominal reference input along with the robot's current state ($\q,\qd$) and selects joint torques $\btau$ and an input to the reference system $\u_r$. This method of selecting $\u_r$ allows us to automatically avoid task-space configurations that inevitably lead to singularities, like those shown in Figures \ref{fig:singular_config_1} and \ref{fig:singular_config_3}.

To formulate this QP-based controller, note that $\Jbar^T\btau = \f^{des}$, where $\f^{des}$ is given by (\ref{eq:controller}), is not strictly necessary for passivity. All we really need for passivity is for $\dot{V} \leq 0$ in the absence of external disturbances, which ensures that $\dot{V} \leq \xderr^T\f_{ext}$ in the presence of disturbances. 

Furthermore, note that $\dot{V}$ is linear in $\qdd$ (\ref{eq:vdot}). This allows us to add a constraint on $\dot{V}$ in a QP as follows:
\begin{align*}
   \min_{\btau,\qdd} ~& \|\Jbar^T\btau - \f^{des}\|^2 \\
   \text{s.t. } & \M\qdd + \C\qd + \btau_g = \btau \\
                & \dot{V}(\qdd) \leq 0 \\
                & \text{Singularity Avoidance }(\ref{eq:our_cbf_constraint}).
\end{align*}
This optimization, which is a convex QP, ensures that both passivity and singularity avoidance hold for any solution. There is a problem, however: the singularity avoidance constraint (\ref{eq:our_cbf_constraint}) and the passivity constraint $\dot{V} \leq 0$ may be in conflict, leading to an infeasible QP.

To avoid this issue, we treat the input to the reference system $\u_r$, as an optimization variable. This is inspired by a similar approach in our prior work on whole-body control for humanoid walking\cite{kurtz2020approximate}. Note that (\ref{eq:vdot}) is also linear in $\u_r = \xdd_r$: thus $\dot{V}$ is still linear in the decision variables. 

Our proposed optimization-based controller is thus given by
\begin{equation}\label{eq:our_qp}
\begin{aligned}
   \min_{\btau,\qdd,\u_r} ~& w_1\|\u_r - \u_r^{nom}\|^2 + w_2\|\Jbar^T\btau - \f^{des}\|^2 \\
   \text{s.t. } & \M\qdd + \C\qd + \btau_g = \btau \\
                & \dot{V}(\qdd, \u_r) \leq 0 \\
                & \text{Singularity Avoidance }(\ref{eq:our_cbf_constraint}).
\end{aligned}
\end{equation}
The scalar, positive weights $w_1$ and $w_2$ regulate the relative priorities of tracking the nominal input to the reference system and applying the standard PBC controller given by (\ref{eq:controller}). 

We assume that the resulting controller is continuous:
\begin{assumption}\label{assumption:lipschitz}
    The control law generated by sequentially solving (\ref{eq:our_qp}) is locally Lipschitz.
\end{assumption}
This standard continuity assumption is necessary for ensuring forward invariance with ECBFs \cite{ames2019control}.

Our controller has several desirable properties. First, it is always feasible:
\begin{proposition}\label{prop:qp_feasibility}
    The quadratic program (\ref{eq:our_qp}) has a feasible solution for any non-singular joint configuration $\q$.
\end{proposition}
\begin{proof}
    From Proposition \ref{prop:ecbf}, we know that $\qdd$ and $\btau$ can always be found such that the dynamics constraint (\ref{eq:wholebody_dynamics}) and the singularity avoidance constraint (\ref{eq:our_cbf_constraint}) both hold. We now show that for any $(\q, \qd, \qdd)$, $\u_r$ can be selected such that $\dot{V} \leq 0$. 
    
    Recall from (\ref{eq:vdot}) that if $\xderr = 0$, then $\dot{V} = 0$ and the constraint holds. If $\xderr \neq 0$, then $\u_r$ enters only in the term
    \begin{equation*}
        -\xderr^T \LL \u_r.
    \end{equation*}
    Together with the fact that $\LL$ is positive definite, this means that $\u_r$ can always be selected such that $\dot{V} \leq 0$, regardless of the other terms in $\dot{V}$ which depend on $\q,\qd,\qdd$.
\end{proof}

Furthermore, any satisfying solution to this optimization enforces both passivity and singularity avoidance, as shown in the following propositions:
\begin{proposition}\label{prop:qp_passivity}
    For $\btau$ selected as solutions of (\ref{eq:our_qp}), the closed-loop system with input $\f_{ext}$ and output $\xderr$ is passive under the storage function (\ref{eq:storage function}).
\end{proposition}
\begin{proof}
    This proposition follows trivially from the inclusion of the constraint $\dot{V}(\qdd,\u_r) \leq 0$.
\end{proof}
\begin{proposition}\label{prop:qp_singularity_avoidance}
    The controller (\ref{eq:our_qp}) renders the set $\mu(\q) \geq \epsilon$ forward invariant.
\end{proposition}
\begin{proof}
    This proposition follows trivially from the inclusion of the constraint (\ref{eq:our_cbf_constraint}), Assumption \ref{assumption:lipschitz}, and Proposition \ref{prop:ecbf}. 
\end{proof}

Note that while our focus here is on singularity avoidance, a similar strategy could also be employed to enforce other constraints, such as joint angle limits \cite{cortez2020correct} or contact/friction constraints \cite{wieber2016modeling}, along with passivity. 

\section{Simulation Results}\label{sec:simulation}

We illustrate our proposed control approach in simulation using a 7-DoF model of the Kinova Gen3 Robot Arm. We use Drake \cite{drake} and python for simulation and dynamics computations. Code is available at \cite{github} and an online interactive demonstration can be found at \cite{colab}. 

We consider the task-space state $\x \in \mathbb{R}^6$ to consist of the position and orientation (Euler angles) of the end-effector. The reference controller used to generate $\u_r^{nom}$ is a simple PD controller that attempts to guide the end-effector to a desired pose $\x_r^{des}$:
\begin{equation}\label{eq:reference_controller}
    \u_r^{nom} = -\K_P^{ref}\left(\x_r - \x_r^{des}\right)  - \K_D^{ref}\xd_r,
\end{equation}
where we used values of $\K_P^{ref} = \K_D^{ref} = 2\bm{I}$.

\begin{figure}
    \centering
    \begin{subfigure}{0.48\linewidth}
        \centering
        \includegraphics[width=\linewidth]{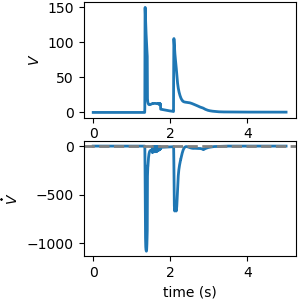}
    \end{subfigure}
    \begin{subfigure}{0.48\linewidth}
        \includegraphics[width=\linewidth]{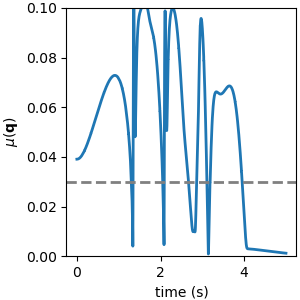} 
    \end{subfigure}
    \caption{Unconstrained PBC: storage function $V$ (\ref{eq:storage function}), storage function derivative $\dot{V}$ (\ref{eq:vdot}), and manipulability index $\mu$ (\ref{eq:manipulability}) over time under the standard unconstrained passivity-based controller (\ref{eq:controller}). Passivity is guaranteed ($\dot{V} \leq 0$), but the robot takes extreme motions in near-singular configurations, resulting in significantly degraded performance.}
    \label{fig:plots_unconstrained}
\end{figure}

\begin{figure}
    \centering
    \begin{subfigure}{0.48\linewidth}
        \centering
        \includegraphics[width=\linewidth]{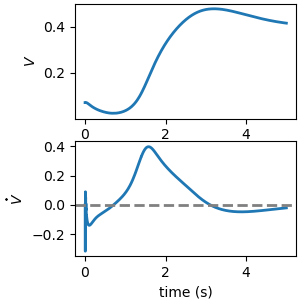}
    \end{subfigure}
    \begin{subfigure}{0.48\linewidth}
        \includegraphics[width=\linewidth]{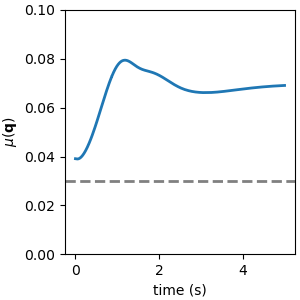} 
    \end{subfigure}
    \caption{Unconstrained PBC with damped Jacobian pseudoinverse: storage function $V$ (\ref{eq:storage function}), storage function derivative $\dot{V}$ (\ref{eq:vdot}), and manipulability index $\mu$ (\ref{eq:manipulability}) over time under the standard unconstrained passivity-based controller (\ref{eq:controller}) with a damped Jacobian pseudoinverse. The robot avoids singularity $(\mu > 0)$, but passivity is not guaranteed ($\dot{V} > 0$).}
    \label{fig:plots_dls}
\end{figure}

We compare our proposed approach (\ref{eq:our_qp}) with unconstrained PBC (\ref{eq:controller}) and standard constrained PBC (\ref{eq:standard_pbc_qp}). In our controller and the standard constrained approach, we used $\epsilon = 0.03$ and solved the QPs at roughly 300Hz using the OSQP solver \cite{osqp}. For our controller, we use weights of $w_1 = 1$ and $w_2 = 10$. 

\begin{remark}
    Our proposed controller requires computing several standard quantities, such as $\M,\C\qd,\btau_g$, and $\J$, for which efficient recursive algorithms are available \cite{featherstone2014rigid}, as well as some less standard quantities like $\C$, $\dot{\J}$, and $\dot{\J}_\mu$. We use Drake's automatic differentiation features to derive these quantities (finite differences in the case of $\dot{\J}_\mu$), though more efficient algorithms for some of these quantities do exist \cite{echeandia2020numerical}.
\end{remark}

Starting from the same initial condition, we used (\ref{eq:reference_controller}) to regulate the reference system to a desired end-effector pose ($\x_r^{des}$) outside the reachable workspace of the robot. Tracking this reference requires entering a singular configuration.

Plots of the storage function $V$ and manipulability index $\mu(\q)$ over time for each approach are shown in Figures \ref{fig:plots_unconstrained}-\ref{fig:plots_ours}. For the unconstrained PBC controller (\ref{eq:controller}, Figure \ref{fig:plots_unconstrained}), passivity is guaranteed but singularity avoidance is not. This is demonstrated by the fact that $\dot{V} \leq 0$ over the whole trajectory, but $\mu(\q)$ comes close to zero at several points. In these near-singular configurations, the controller (\ref{eq:controller}) requires extreme joint torques to keep $\dot{V} \leq 0$, leading to extreme motions of the robot arm and the corresponding jumps in $V$. These jumps are possible because the controller is applied in discrete-time, leading to a discrepancy between $V$ as computed by (\ref{eq:storage function}) and $\dot{V}$ as computed by (\ref{eq:vdot}) when $\btau$ is very large.

We also consider unconstrained PBC with a damped Jacobian pseudoinverse (\ref{eq:damped_jacobian_pinv}), using damping constant $\delta = 0.001$. This approach is shown in Figure \ref{fig:plots_dls}. This is a common method of ensuring feasibility and numerical stability in singular/near-singular configurations, but comes at the price of degraded performance. Passivity is not guaranteed under this approach, as evidenced by positive values of $\dot{V}$.

\begin{figure}
    \centering
    \begin{subfigure}{0.48\linewidth}
        \centering
        \includegraphics[width=\linewidth]{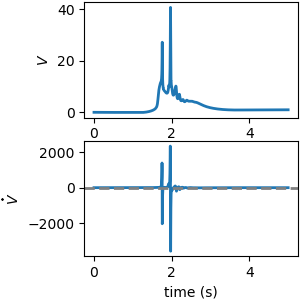}
    \end{subfigure}
    \begin{subfigure}{0.48\linewidth}
        \includegraphics[width=\linewidth]{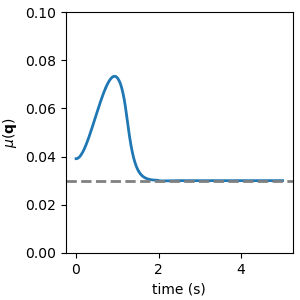} 
    \end{subfigure}
    \caption{Standard constrained PBC: storage function $V$ (\ref{eq:storage function}), storage function derivative $\dot{V}$ (\ref{eq:vdot}), and manipulability index $\mu$ (\ref{eq:manipulability}) over time under the standard constrained passivity-based control scheme (\ref{eq:standard_pbc_qp}). Singularity avoidance is guaranteed ($\mu \geq \epsilon$), but passivity is not.}
    \label{fig:plots_constrained}
\end{figure}

The standard constrained PBC method (\ref{eq:standard_pbc_qp}, Figure \ref{fig:plots_constrained}), guarantees singularity avoidance but not passivity. This is shown by the fact that $\mu(\q)$ stays above $\epsilon$ (grey dashed line). When the singularity avoidance constraints become active, however, the passivity properties are lost. This is demonstrated by positive values of $\dot{V}$ and corresponding increases in $V$.

\begin{figure}
    \centering
    \begin{subfigure}{0.48\linewidth}
        \centering
        \includegraphics[width=\linewidth]{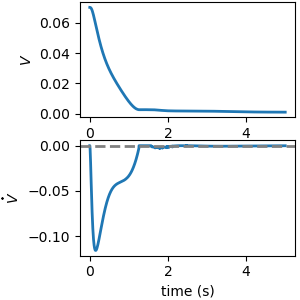}
    \end{subfigure}
    \begin{subfigure}{0.48\linewidth}
        \includegraphics[width=\linewidth]{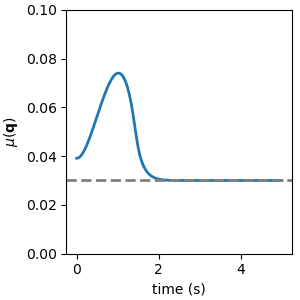} 
    \end{subfigure}
    \caption{Proposed approach: storage function $V$ (\ref{eq:storage function}), storage function derivative $\dot{V}$ (\ref{eq:vdot}), and manipulability index $\mu$ (\ref{eq:manipulability}) over time under our proposed approach. Both passivity ($\dot{V} \leq 0$) and singularity avoidance ($\mu \geq \epsilon$) are enforced.}
    \label{fig:plots_ours}
\end{figure}

Finally, our proposed method (Figure \ref{fig:plots_ours}) guarantees both passivity and singularity avoidance. Both $\dot{V} \leq 0$ and $\mu(\q) \geq \epsilon$ hold throughout the trajectory. This is possible because we modify the reference input $\u_r$ as the system approaches singularity, so the reference end-effector pose $\x_r$ never leaves the robot's reachable workspace. 

\section{Conclusion}\label{sec:conclusion}

In this paper, we proposed a new optimization-based strategy for constrained PBC with guaranteed feasibility, passivity, and singularity avoidance. Our key insights are to constrain the evolution of the storage function ($\dot{V}$) in a convex QP and to modify the input to a reference system when necessary. While we focus on singularity avoidance, the proposed methods can also be applied to other constraints such as joint limits and contact/friction constraints. Future work will focus on hardware implementation, extensions to legged locomotion, and principled methods of including multiple constraints. 

\bibliographystyle{IEEEtran}
\bibliography{references}

\end{document}